\documentclass[a4paper,11pt,oneside]{article}

\usepackage{amsmath,amssymb}
\usepackage{epsfig}
\usepackage[utf8]{inputenc}
\usepackage{textcomp}
\usepackage{algorithm}
\usepackage{algpseudocode}
\usepackage{color}
\usepackage{graphicx}
\usepackage{caption} 
\usepackage{subcaption}
\usepackage{bm}
\usepackage{amsthm}
\usepackage{mathrsfs}
\usepackage{dsfont}
\usepackage{algcompatible}
\usepackage{url}

\usepackage{geometry}
\newtheorem{thm}{Theorem}[section]
\newtheorem{cor}[thm]{Corollary}
\newtheorem{definition}[thm]{Definition}

\newtheorem{proper}[thm]{Property}

\newtheorem{remark}[thm]{Remark}

\title{\textbf{\textsc{Data augmentation and feature selection for automatic model recommendation in computational physics}}}
\author{Thomas DANIEL\footnotemark[1]  \footnotemark[2] , Fabien CASENAVE\footnotemark[1] , Nissrine AKKARI\footnotemark[1] , David RYCKELYNCK\footnotemark[2]}
\date{January 12, 2021}

\DeclareMathOperator*{\argmax}{arg\,max}

\newcommand{\bs}[1]
{
\boldsymbol{ #1}
}

\newcommand{\Prob}{\mathbb{P}}

\newcommand{\E}{\mathbb{E}}

\algnewcommand\INPUT{\item[\textbf{Input:}]}%
\algnewcommand\OUTPUT{\item[\textbf{Output:}]}%

\algdef{SE}[SUBALG]{Indent}{EndIndent}{}{\algorithmicend\ }%
\algtext*{Indent}
\algtext*{EndIndent}

\begin{document}
\maketitle

\renewcommand{\thefootnote}{\fnsymbol{footnote}}
\footnotetext[1]{ \ SafranTech, Rue des Jeunes Bois, Ch\^ateaufort, 78114 Magny-les-Hameaux (France).}
\footnotetext[2]{ \ MINES ParisTech, PSL University, Centre des mat\'{e}riaux (CMAT), CNRS UMR 7633, BP 87, 91003 Evry (France).}
\renewcommand{\thefootnote}{\arabic{footnote}}

\section*{Abstract}
Classification algorithms have recently found applications in computational physics for the selection of numerical methods or models adapted to the environment and the state of the physical system. For such classification tasks, labeled training data come from numerical simulations and generally correspond to physical fields discretized on a mesh. Three challenging difficulties arise: the lack of training data, their high dimensionality, and the non-applicability of common data augmentation techniques to physics data. This article introduces two algorithms to address these issues, one for dimensionality reduction via feature selection, and one for data augmentation. These algorithms are combined with a wide variety of classifiers for their evaluation. When combined with a stacking ensemble made of six multilayer perceptrons and a ridge logistic regression, they enable reaching an accuracy of $90\%$ on our classification problem for nonlinear structural mechanics.

\noindent \textbf{Keywords:} machine learning, classification, automatic model recommendation, feature selection, data augmentation, numerical simulations.

\section{Introduction}

Classification problems can be encountered in various disciplines such as handwritten text recognition~\cite{doi:10.1162/neco.1989.1.4.541}, document classification~\cite{DocClassif}, and computer-aided diagnosis in the medical field~\cite{KOUROU20158}, among many others. In numerical analysis, classification algorithms are getting more and more attention for the selection of efficient numerical models that can predict the behavior of a physical system with very different states or under various configurations of its environment~\cite{LDEIM2014, RyckelynckComputerVision, 10.3389/fmats.2019.00075, Maulik19, doi:10.2514/6.2020-0418, kapteyn2020physicsbased, MAULIK2020132409, ROM-net}. Classifiers have been used as reduced-order model (ROM) selectors in~\cite{LDEIM2014, RyckelynckComputerVision, doi:10.2514/6.2020-0418, kapteyn2020physicsbased, ROM-net} in computational mechanics, enabling the computation of approximate solutions at lower cost by replacing a generic high-fidelity numerical model by a specific (or local) ROM adapted to the simulation's context. Reduced-order modeling~\cite{10.5555/2568435, keiper2018reduced} consists in identifying an appropriate low-dimensional subspace on which the governing equations are projected in order to reduce the number of degrees of freedom of the solution. In~\cite{ROM-net}, the combination of a classifier with a dictionary of local ROMs has been termed \textit{dictionary-based ROM-net}. Such approaches are promising numerical methods using both physics equations and a collection of latent spaces to compute approximations of solutions lying in nonlinear manifolds.

Dictionary-based ROM-nets use a physics-informed automatic data labeling procedure based on the clustering of numerical simulations. Due to the cost of numerical simulations, training examples for classification are limited in number. Moreover, the dimensionality of input data can be very high, especially when dealing with physical fields discretized on a mesh (finite-difference methods~\cite{smith1985numerical}, finite-element method~\cite{ern2013theory}, finite-volume method~\cite{versteeg2007introduction}) or with bond graphs modeling engineering systems~\cite{borutzky2011bond}.

When classification data are high-dimensional, dimensionality reduction techniques can be applied to reduce the amount of information to be analyzed by the classifier. For classification problems where the dimension of the input data is higher than the number of training examples, dimensionality reduction is crucial to avoid overfitting. In addition, when considering physical fields discretized on a mesh, the dimension of the input space can reach $10^6$ to $10^8$ for industrial problems. In such cases, the input data are too hard to manipulate, which dramatically slows down the training process for the classifier and thus restrains the exploration of the hyperparameters space, as it requires multiple runs of the training process with different values for the hyperparameters. Applying data augmentation techniques to increase the number of examples in the training set is also impossible, as it would cause memory problems. Therefore, dimensionality reduction is recommended not only for reducing the risk of overfitting, but also for facilitating the training phase and enabling data augmentation. 

\textit{Feature selection}~\cite{CHANDRASHEKAR201416} aims at decreasing the number of features by selecting a subset of the original features. It differs from \textit{feature extraction}, where new features are created from the original ones (e.g. Principal Component Analysis, PCA, and more generally encoders taken from undercomplete autoencoders~\cite{Goodfellow-et-al-2016}). Feature selection can be seen as applying a mask to a high-dimensional random vector to get a low-dimensional random vector containing the most relevant information. It is preferred over autoencoders when interpretability is important~\cite{janecek08}. Furthermore, contrary to undercomplete autoencoders trained with the mean squared error loss, most feature selection algorithms do not intend to find reduced features enabling the reconstruction of the input: features are selected for the purpose of predicting class labels, which makes these algorithms more goal-oriented for supervised learning tasks.

Among the existing feature selection algorithms, univariate filter methods consist in computing a score for each feature and ranking the features according to their scores. The score measures how relevant a feature is for the prediction of the output variable. If $N_{f}$ is the target number of features, then the $N_f$ features with the highest scores are selected, and the others are discarded. The major drawback of univariate filter methods is that they do not account for relations between the selected features. The resulting set of selected features may then contain redundant features. To address this issue, the \textit{minimum redundancy maximum relevance} (mRMR) algorithm~\cite{mRMR2003, mRMR2005} tries to find a tradeoff between relevance and redundancy. However, for very large numbers of features like in computational physics, evaluating the redundancy is very computationally demanding. Fortunately, working on physics data provides other possibilities to define a redundancy measure. In this paper, we propose a new feature selection algorithm suitable for features coming from the same physical quantity but corresponding to different points in a space-time discretization. It is assumed that this physical quantity, defined as a function of space and/or time, has some smoothness properties. This is often the case in physics, where the physical quantity satisfies partial differential equations and boundary conditions. In~\cite{AFST_1989_5_10_2_325_0}, it is shown that the solution of Poisson's equation on a Lipschitz domain in $\mathbb{R}^3$ with a $L^2$ source term and Dirichlet or Neumann boundary conditions is continuous. Poisson's equation is well-known in physics, and can be found for example in electrostatics, in Gauss's law for gravity, in the stationary heat equation, and in the stationary particle diffusion equation. If the features of a random vector contain the discretized values of a smooth function of space and time, then their correlations are related to their proximities on the space-time grid. The approach presented in this paper is depicted as a geostatistical variant of mRMR algorithm, in the sense that it consists in modeling the redundancy as a function of space and time.

Once the dimension of the input space is reduced, another challenge of the classification problems encountered in computational physics must be addressed: the lack of training data. \textit{Data augmentation} refers to techniques aiming at enlarging the training set by generating new examples from the original ones. For image classification, many class-preserving operations can be used to create new images, such as translations, rotations, cropping, scaling, and changes in colors, brightness and contrast. Unfortunately, these common techniques cannot be used when considering physics data. For this type of data, new examples can be generated using generative adversarial networks (GAN~\cite{GenerativeAdversarialNetworks}, see~\cite{AkkariGAN2020} for the use of deep convolutional GANs in computational fluid dynamics). However, training GANs is quite complex in practice and may also be made more difficult by the lack of training examples. More simply, new data can be generated by convex combinations of the original examples. SMOTE~\cite{SMOTE} takes convex combinations of input data with their nearest neighbors in the input space. ADASYN~\cite{ADASYN} uses the same idea but focuses more on examples that are hard to learn, \textit{i.e.} those having examples of a foreign class in their neighborhoods. Both data augmentation algorithms use k-nearest neighbors algorithm and thus compute Euclidean distances in the input space. When working on high-dimensional physics data, this approach may suffer from the curse of dimensionality~\cite{Bellman61}. In addition, defining neighborhoods with the Euclidean distance in the input space is not always appropriate, since dictionary-based ROM-nets use physics-aware dissimilarities to label the data, such as distances on the primal variable or on a quantity of interest. The data augmentation algorithm developed in this article consists in growing sets around original examples by incrementally adding nearest neighbors in terms of the dissimilarity measure used for the automatic data labeling procedure. These sets are used to generate new data by convex combinations. Contrary to SMOTE and ADASYN, the risk of generating new data with wrong labels is controlled by checking that the convex hulls of the growing sets do not contain any example belonging to a foreign class.

In sum, the contributions of this paper are motivated by difficulties encountered in our previous work on ROM-nets~\cite{ROM-net}. These difficulties are inherent to classification tasks on simulation data and can be summarized in three main issues:
\begin{itemize}
\item the lack of training data due to the expensive data labeling procedure involving simulations with a high-fidelity model (risk of overfitting);
\item the high dimensionality of input data (risk of overfitting);
\item most common data augmentation techniques are not applicable to physics data. 
\end{itemize}
The feature selection and data augmentation strategies introduced in this paper are developed to tackle these difficulties. Classification problems encountered in computational physics are described in Section~\ref{SectionClassification}. Section~\ref{SectionDefOurPb} presents the classification problem studied in this paper. The feature selection algorithm is described in Section~\ref{SectionFS} and is shown to efficiently remove irrelevant and redundant features. Section~\ref{SectionDA} presents the data augmentation algorithm, which successfully generates a large amount of new data with correct labels. Finally, Section~\ref{SectionValidation} evaluates both algorithms in conjunction with 14 different classifiers. On our classification task, the average accuracy gain due to data augmentation is $4.98\%$. Using ensemble methods on classifiers combined with our algorithms enables reaching a classification accuracy of $90\%$.

\section{Classification in the context of numerical modeling}
\label{SectionClassification}

\subsection{Classification: a brief review}

Supervised learning is the task of learning the correspondence between input data $X$ and outputs $Y$ from a training set of input-output pairs $\{ (x_{i},y_{i})\}_{1 \leq i\leq N}$. Supervised machine learning problems fall into two categories: regression problems, for which the outputs take continuous values, and classification problems, consisting in the prediction of categorical labels. This paper focuses on the latter, with the additional assumptions that $X$ is a continuous multivariate random variable having a probability density function $p_{X} : \mathcal{X} \rightarrow \mathbb{R}_{+}$, and that any observation $x\in\mathcal{X}$ is associated to a single label $y$. The discrete random variable $Y$ follows a categorical distribution (or multinoulli distribution) whose probability mass function is defined by:
\begin{equation}
\forall y \in \mathbb{R}, \quad p_{Y}(y) = \sum_{k=1}^{K} \Prob_{Y}(k) \delta (y-k)
\label{PMFY}
\end{equation}
\noindent where $K$ is the number of categories (or classes), $\delta$ is the Dirac delta function, and $\Prob_{Y}(k)$ denotes the probability of the event $Y=k$ for a given label $k \in [\![ 1;K ]\!]$. The labeled training data are drawn from the joint probability distribution $p_{X,Y}$, called the \textit{data-generating distribution}. As $X$ is continuous and $Y$ is discrete, $p_{X,Y}$ is a mixed joint density and can be obtained with the formula:\begin{equation} p_{X,Y}(x,y) = p_{Y} (y) \ p_{X\vline Y}(x \ | \ y) = \sum_{k=1}^{K} \Prob_{Y}(k) \delta (y-k) p_{X\vline Y}(x \ | \ y)
\label{jointDensity}
\end{equation} \noindent with $p_{X\vline Y}$ being the class-conditional probability distribution.

In the present paper, we are interested in single-label multiclass problems. Hence, the classification problem considered here reads: \textit{given an integer $K \geq 2$ and a training set $\{ (x_{i},y_{i})\}_{1 \leq i\leq N} \subset \mathcal{X}\times [\![ 1;K ]\!]$, train a classifier $\mathcal{C}( \ . \ ; \theta):\mathcal{X} \rightarrow [\![ 1;K ]\!]$ to assign any observation $x\in\mathcal{X}$ to the correct class, with $\theta$ denoting the parameters of the classifier}. However, reaching the highest possible accuracy on the training set is not the objective to be pursued, since it usually leads to \textit{overfitting}. Indeed, the classifier is supposed to be applied to new unseen data, or \textit{test} data, after the training phase. Therefore, the generalization ability of the classifier is at least as important as its performance on the training set. A classifier with high capacity\footnote{Ability to learn classes with complex boundaries (related to model complexity).} perfectly fits training data but is very sensitive to noise, leading to high test error and thus overfitting. On the other hand, a classifier with low capacity can produce smaller error gaps between training and test predictions, but such a classifier may not be able to fit the data, which is called \textit{underfitting}. This dilemma is known as the \textit{bias-variance tradeoff}: low model capacity leads to high bias, while high model capacity leads to high variance.

For a given observation $x\in\mathcal{X}$, probabilistic classification algorithms estimate the membership probabilities $\Prob_{\textrm{model}} \left( y \ \vline \ x ; \theta\right)$ for each class $y\in [\![ 1;K ]\!]$. The classifier $\mathcal{C}$ returns the index of the class with the highest membership probability:
\begin{equation}
\mathcal{C}(x ; \theta) = \argmax_{y\in[\![1;K]\!]} \left( \Prob_{\textrm{model}} \left( y \ \vline \ x ; \theta\right) \right)
\end{equation}

\noindent The parameters $\theta$ must be optimized to minimize the \textit{expected risk} $\mathcal{J}(\theta)$ defined by:
\begin{equation}
\mathcal{J}(\theta) = \E_{(X,Y)\sim p_{X,Y}} \left[ L \left( \mathcal{C}(X ; \theta),Y \right) \right]
\end{equation}
\noindent where $L$ is the per-example loss function quantifying the error between the predicted class $\mathcal{C}(X ; \theta)$ and the true class $Y$. However, as the true data-generating distribution $p_{X,Y}$ is unknown, the expected risk must be estimated by computing the expectation with respect to the empirical distribution $\widehat{p}_{X,Y}$:
\begin{equation}
\widehat{p}_{X,Y} (x,y) = \frac{1}{N} \sum_{i=1}^{N} \delta (x - x_i , y - y_i )
\end{equation}
\noindent Therefore, the training process consists in minimizing the \textit{empirical risk}:
\begin{equation}
\widehat{\mathcal{J}}(\theta) = \E_{(X,Y)\sim \widehat{p}_{X,Y}} \left[ L \left( \mathcal{C}(X;\theta),Y \right) \right] = \frac{1}{N} \sum_{i=1}^{N} L \left( \mathcal{C}(x_{i};\theta),y_{i} \right)
\end{equation}
This is known as the \textit{empirical risk minimization} (ERM) principle~\cite{Vapnik1998}. Common choices for the function $L$ are the hinge loss (defined for multiclass problems in~\cite{10.5555/944790.944813}) used by support vector machines (SVMs), and the log loss or negative log-likelihood:
\begin{equation}
L \left( \mathcal{C}(x;\theta),y \right) = - \log \ \Prob_{\textrm{model}} \left( y \ \vline \ x ; \theta\right)
\end{equation}
\noindent widely used for classifiers based on artificial neural networks (ANNs) and for logistic regression. When $L$ is the negative log-likelihood, the objective function $\widehat{\mathcal{J}}(\theta)$ is the cross-entropy loss and the optimal set of parameters $\theta^{*}$ minimizing $\widehat{\mathcal{J}}$ is the maximum likelihood estimator~\cite{Hastie2005TheEO}. Usually, a regularization term is added to the empirical risk to penalize the model complexity in order to reduce overfitting.

The boundaries between classes in the input space are called decision boundaries. Linear classifiers are classification algorithms for which the decision boundaries are defined by linear combinations of the features of $X$. Linear classifiers are appropriate when the classes are linearly separable in $\mathcal{X}$, which means that the decision boundaries correspond to portions of hyperplanes. Linear classifiers include logistic regression~\cite{10.2307/2280041, 10.2307/2983890, multinomialLR}, linear discriminant analysis (LDA~\cite{Hastie2005TheEO}), and the linear support vector classifier (linear SVM~\cite{10.1023/A:1022627411411, OptimalMarginClassifier}).

Many algorithms exist for nonlinear classification problems, each of them having its own advantages and drawbacks. As a kernel method, the linear SVM is extended to nonlinear classification problems using the \textit{kernel trick} based on Mercer's theorem~\cite{doi:10.1098/rsta.1909.0016}. Artificial neural networks~\cite{Ivakhnenko1966CyberneticPD, Joseph1961} (see~\cite{Schmidhuber2015DeepLI} for a historical review) have become very popular due to their performances in numerous classification contests. Decision trees (e.g. CART algorithm~\cite{Breiman1983ClassificationAR}) and naive Bayes classifiers~\cite{NaiveBayes1, NaiveBayes2} are well-known for their interpretability. Other nonlinear classifiers include the k-nearest neighbors algorithm (kNN~\cite{1053964}), and quadratic discriminant analysis (QDA~\cite{Hastie2005TheEO}). In~\cite{ComparisonSupervisedLearningAlgos}, the most common classifiers are compared on eleven binary classification problems. Short reviews of classification algorithms can be found in~\cite{ReviewClassificationKotsiantis, ReviewClassificationPerezOrtiz}.

Usually, combining several models to form a meta-estimator results in more robust predictions and reduces overfitting. This idea is used in ensemble methods such as bagging (or bootstrap aggregating)~\cite{Breiman1996BaggingP}, feature bagging (or random subspace method)~\cite{Ho1998TheRS}, stacking~\cite{Hastie2005TheEO, Wolpert1992StackedG}, boosting (including the well-known AdaBoost algorithm~\cite{Freund1995ADG, Hastie2009MulticlassA}), gradient boosting~\cite{GradientBoosting1, GradientBoosting2, GradientBoosting3, GradientBoosting4}, and voting classifiers based on either a majority vote or a soft vote (technique known as ensemble averaging~\cite{Haykin99}). Random forests~\cite{Breiman2001RandomF} combine bagging and feature bagging to build an ensemble of decision trees.

\subsection{Classification for numerical simulations}

Classification algorithms have recently found applications in numerical simulations, and more specifically for the selection of numerical models adapted to the context of the simulation. In this case, the class labels are used to identify the models.

Applications to turbulence modeling in computational fluid dynamics can be found in~\cite{Maulik19,MAULIK2020132409}. In large eddy simulations (LES, see~\cite{meneveau2006large}), the Navier-Stokes equations are filtered to avoid resolving small-scale turbulent structures whose effects are taken into account either by sub-grid scale models (explicit LES closures) or via the dissipation induced by numerical schemes (implicit LES). In~\cite{Maulik19}, sub-grid statistics obtained from direct numerical simulations enable training a fully-connected deep neural network to switch between different explicit LES closures at any point of the grid. This classifier is reused in~\cite{MAULIK2020132409}, this time for switching between different numerical schemes in implicit LES. In both cases, the classifier is used to increase the accuracy of numerical predictions.

The idea of locally switching between different simulation strategies can also be found in~\cite{10.3389/fmats.2019.00075} for the multiscale modeling of composite materials. In the multilevel finite-element method ($\textrm{FE}^2$~\cite{FEYEL1999344}), the quantities of interest at every integration point of the macroscopic finite-element mesh are given by a microscopic finite-element computation of an elementary cell representing the material's microstructure. The multi-fidelity surrogate model presented in~\cite{10.3389/fmats.2019.00075} relies on two surrogate models replacing the microscopic finite-element model, namely a reduced-order model taken from~\cite{FRITZEN2018201} and an artificial neural network based regression model. At each integration point of the macroscopic mesh, the classifier (a fully-connected network) analyzes the effective strains and predicts whether the error of the regression model would be acceptable, enabling the selection of either the purely data-driven regression model or the more sophisticated physics-driven ROM. This time, automatic model recommendation by a classifier is used to adapt the model complexity and reduce the computation time.

In~\cite{doi:10.2514/6.2020-0418, kapteyn2020physicsbased}, optimal classification trees (OCTs~\cite{OCT}) are used  as \textit{model selectors} in a data-driven physics-based digital twin of an unmanned aerial vehicle (UAV). The OCTs enable the update of the digital twin according to sensor data by selecting a model from a predefined model library. In this context, the training procedure for the classifier corresponds to an inverse problem. Indeed, training examples are generated by running simulations with all the models in the library and evaluating their predictions at the sensors' locations. Hence, for a given model $y\in [\![ 1;K ]\!]$, the data $x$ are obtained by means of numerical simulations performed with $y$. This corresponds to the forward mapping. The classifier must learn the inverse mapping giving $y$ as a function of $x$. In this example, data labeling is straightforward: the label of a training example $x$ is given by the index $y$ of the model which was used to generate $x$. It is also noteworthy that generating training examples is not too expensive, because numerical simulations are performed with reduced-order models obtained by the Static-Condensation Reduced-Basis-Element
method (SCRBE~\cite{SCRBE1, SCRBE2, SCRBE3, SCRBE4}). In this application, automatic model recommendation gives the UAV the ability to dynamically evaluate its flight capability and replan its mission accordingly.

Another example of classifier used to accelerate numerical simulations can be found in~\cite{LDEIM2014}. Contrary to~\cite{doi:10.2514/6.2020-0418, kapteyn2020physicsbased}, the data labeling procedure relies on the clustering of simulation data. In this framework, the model library is made of cluster-specific DEIM\footnote{Discrete Empirical Interpolation Method.}~\cite{DEIM} models that are faster than the high-fidelity model. The high-fidelity model computes a prediction $u_i$ for each input $x_i$ in the database $\{x_{i}\}_{1 \leq i \leq N}$, resulting in a dataset $\{u_{i}\}_{1 \leq i \leq N}$ on which a clustering algorithm is applied. The predicted variable $u$ is the discretization of a continuous field on a finite-element mesh, thus living in a high-dimensional space. To avoid the so-called \textit{curse of dimensionality}~\cite{Bellman61}, a DEIM-based feature selection technique is used before applying k-means clustering~\cite{kmeans}. Alternatively, the clusters can be obtained with a variant of k-means using the DEIM residual as clustering criterion. Then, for a given training example $x_i$, the class label $y_i$ is defined by the index of the cluster that $u_i$ is assigned to. In the exploitation phase, when dealing with test data, the best DEIM model is selected by a nearest neighbor classifier. The input data given to the classifier are either parameters of the problem or the variable $u$ obtained at the previous time increment. A similar methodology is described in~\cite{RyckelynckComputerVision}, where the concept of model library is termed \textit{model dictionary}, which is the terminology adopted in this paper. The model dictionary is made of hyper-reduced-order models~\cite{Ryckelynck2005}, and the input data $\{x_{i}\}_{1 \leq i \leq N}$ are images of a mechanical experiment. The dimensionality of simulation data is reduced by Principal Component Analysis (PCA) before using k-means clustering. A convolutional neural network~\cite{reviewCNN} is trained to return class labels without computing the intermediate variable $u$ in order to avoid time-consuming operations. This classifier is an approximation of the \textit{true classifier} $\mathcal{K}$ returning the correct label for any input $x$.

\section{Definition of the classification problem}
\label{SectionDefOurPb}

\noindent \textbf{Notations:} the $j$-th feature of a random vector $X$ is the real-valued random variable denoted by $X^j$. Its observations are denoted by $x^j$, or $x^{j}_{i}$ when indexing is necessary, for example when considering training data. When $X$ is obtained by discretizing a random field on a mesh, the feature $X^j$ corresponds to the value taken by the random field at the $j$-th node. In the numerical application presented in this work, a random temperature field is considered. The spatial coordinates of the $j$-th node are stored in a vector $\bs{\xi}_j \in \mathbb{R}^3$. The categorical variable $Y$ indicates which model should be used. \qed

In this paper, input data $\{x_{i}\}_{1 \leq i \leq N}$ correspond to several instances or \textit{variabilities} of a physical field discretized on a mesh. Let $\mathcal{N}$ be the number of nodes in the mesh. If the physical field is scalar and defined at the nodes, then each observation $x_i$ is a vector of $\mathbb{R}^{\mathcal{N}}$. For relatively small problems, $\mathcal{N}$ is in the order of $10^4$ to $10^5$. For some industrial problems, $\mathcal{N}$ can be in the order of $10^6$ to $10^8$. The dataset $\{x_{i}\}_{1 \leq i \leq N}$ may come from experiments, numerical simulations, statistical models, or a combination of them, and contains from $10^2$ to $10^4$ observations. It is assumed that all features of all observations are known, contrary to some classification tasks in other disciplines encountering the problem of missing values. This assumption is clearly satisfied when data come from numerical simulations or statistical models. For experimental data, numerous techniques provide space-distributed measurements that can be projected onto the mesh, such as particle image velocimetry~\cite{adrian2011particle} in fluid dynamics, digital image correlation~\cite{DIC85} and photoelastic experiments~\cite{doi:10.1063/1.1745316} in solid mechanics, and temperature-sensitive paints~\cite{dlr47882} measuring surface temperatures.

The framework considered in this paper is the same as in~\cite{ROM-net} for ROM-nets, where the input variabilities are supposed to be used for an uncertainty propagation study in a physics problem $\mathcal{P}$, for which a high-fidelity model $m_{HF}$ is available. The physics problem $\mathcal{P}$ is a time-dependent problem. As the high-fidelity model is too computationally expensive, dictionary-based ROM-nets have been introduced to reduce the computation time by means of a reduced-order model dictionary and a classifier playing the role of a model selector. The dictionary-based ROM-net is trained on the available dataset $\{x_{i}\}_{1 \leq i \leq N}$. For a given observation $x_i$, the class label $y_i$ indicates the most appropriate model in the dictionary to be used for fast simulations with limited errors with respect to the high-fidelity model $m_{HF}$. Class labels are obtained by the following data labeling procedure:
\begin{itemize}
\item \textbf{Step 1:} for each observation $x_i$ in the dataset, use the high-fidelity model $m_{HF}$ to solve a simplified version $\mathcal{P}'$ of the physics problem $\mathcal{P}$ (for example, the problem $\mathcal{P}'$ can consist in solving $\mathcal{P}$ for a few time increments only). The primal solution of $\mathcal{P}'$ computed for $x_i$ is denoted by $u_i$. It consists of a collection $\{u_{i}^{n}\}_{1 \leq n \leq n_t}$ of $n_t$ fields defined on the mesh, with $n_t$ being the number of time increments in problem $\mathcal{P}'$.
\item \textbf{Step 2:} given $\{u_{i}\}_{1 \leq i \leq N}$, compute the dissimilarity matrix $\bs{\delta}\in\mathbb{R}^{N \times N}$ with the following formula:$$\delta_{ij} = \delta (x_i , x_j ) = d_{\textrm{Gr}(\infty, \infty)}\left( \textrm{span}(\{u_{i}^{n}\}_{1 \leq n \leq n_t}), \textrm{span}(\{u_{j}^{n}\}_{1 \leq n \leq n_t}) \right)$$ with $d_{\textrm{Gr}(\infty, \infty)}$ being the Grassmann metric defined in~\cite{Grassmann}. The coefficient $\delta_{ij}$ is a dissimilarity measure between $x_i$ and $x_j$.
\item \textbf{Step 3:} k-medoids clustering~\cite{kMedoids1, kMedoids2, kMedoids3} is applied to the dissimilarity matrix $\bs{\delta}$. In this paper, we consider $K=4$ clusters. The label $y_{i} = \mathcal{K}(x_i) \in [\![ 1;K ]\!]$ is given by the index of the cluster containing $u_i$.
\end{itemize}

This procedure gives $N=1000$ examples of input-label pairs $\{(x_{i},y_{i})\}_{1 \leq i \leq N}$. This dataset is split in a training set, a validation set and a test set with cardinalities $600$, $200$ and $200$ respectively, enabling the supervised training and evaluation of a classifier $\mathcal{C}$. For the sake of simplicity, the labeled data are renumbered so that the $N_{\textrm{train}}=600$ first input-output pairs $\{(x_{i},y_{i})\}_{1 \leq i \leq N_{\textrm{train}}}$ form the training set on which the feature selection and data augmentation algorithms presented in this paper are trained.

In this work, the physics problem $\mathcal{P}$ is a temperature-dependent mechanical problem. The structure is made of an elasto-viscoplastic material whose behavior depends on the local value of the temperature field~\cite{CHABOCHE20081642}. The random variable $X$ is a random vector representing the evaluation of the random temperature field on a finite-element mesh containing $\mathcal{N}=42445$ nodes (see Figure~\ref{FEmesh}). The structure is subjected to centrifugal forces and pressure loads. The random temperature fields are generated by a stochastic model described in~\cite{ROM-net}, where ten fluctuation modes are randomly combined and superposed to a reference temperature field. The realizations of the random temperature field are continuous and always satisfy the heat equation. Modeling random fields as random combinations of deterministic spatial functions is quite common when
studying stochastic partial differential equations~\cite{MATTHIES1997283, Sudret2000, MATTHIES20051295}, because a random field can be approximated by truncating its Karhunen-Lo\`eve expansion~\cite{10.1007/s00607-008-0018-3}.

As already stated, the main contributions of this paper are a feature selection strategy and a data augmentation algorithm adapted to the specificities and difficulties of classification problems encountered when training dictionary-based ROM-nets. Concerning feature selection, the main focus is on the fast quantification of features redundancy by taking advantage of the type of input data. Concerning data augmentation, in addition to the constraints that have already been mentioned, it is likely that transforming an input example $x_i$ substantially modifies the intermediate variable $u_i$, and thus the class label $y_i$ might no longer be relevant for the transformed input. Avoiding this situation is crucial to ensure that the augmented data are correctly labeled. Our algorithms are applicable under the assumptions that the random vector $X$ derives from a random field whose realizations are continuous with probability one (sample path continuity, see Definition~2.1 in~\cite{abrahamsen1997review}) and belong to a convex domain~$\mathcal{X}$ related to physics constraints. Lastly, a comparison of various classification algorithms is conducted to put into perspective the choice made in~\cite{ROM-net} to use an ensemble of deep neural networks trained with different architectures and loss functions.

\begin{figure}[!h]
\centering
\includegraphics[scale=0.6]{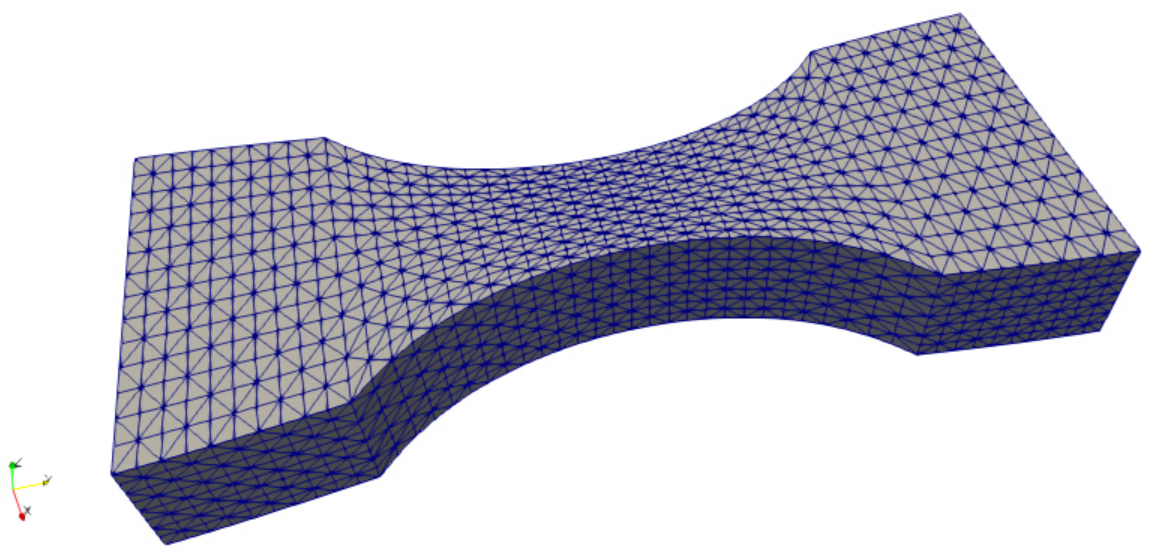}
\caption{Finite-element mesh of the structure considered in this paper.}
\label{FEmesh}
\end{figure}

\begin{remark}
Another strategy would consist in using a regression algorithm for the classification task. Indeed, since our data labeling procedure is based on clustering, the classification problem could be replaced by a regression problem for the prediction of dissimilarities $\{ \delta (x, \tilde{x}_k ) \}_{1 \leq k \leq K}$ for $x\in\mathcal{X}$, with $\tilde{x}_k$ being the medoid of the $k$-th cluster. Given these distances for a new observation $x$, the class label is obtained by taking the integer $k\in [\![ 1;K ]\!]$ associated to the smallest dissimilarity $\delta (x, \tilde{x}_k )$. However, the data augmentation algorithm presented in this paper is not compatible with regression algorithms. For this reason, this paper focuses on classifiers rather than regressors.
\end{remark}

\section{Feature selection}
\label{SectionFS}

\subsection{Feature selection based on mutual information}

We recall that a projection $\pi$ is a linear map satisfying $\pi \circ \pi = \pi$. It is entirely defined by its kernel and its image, which are complementary: given two complementary vector subspaces $V_1$ and $V_2$, there is a unique projection $\pi$ whose kernel is $V_1$ and whose image is $V_2$, namely the projection onto $V_2$ along $V_1$. For more details about projections, see~\cite{Meyer2000}, pages 385 to 388. Let us now give a formal definition of a \textit{feature selector}:

\begin{definition}
\textit{(Feature selector)} Let $V$ be a finite-dimensional real vector space. Given a basis $\mathcal{B} = (e_i)_{1 \leq i \leq \dim(V)}$ of $V$ and a set of integers $S \subset [\![ 1; \dim (V) ]\!]$, the \textit{feature selector} $\pi_{S,\mathcal{B}}: V \rightarrow V$ is the projection whose image is $\emph{\textrm{span}}\left(\{ e_i \}_{i \in S} \right)$ and whose kernel is $\emph{\textrm{span}}\left(\{ e_i \}_{i \in [\![ 1; \dim (V) ]\!] \setminus S} \right)$.
\end{definition}

\noindent When the choice of the basis $\mathcal{B}$ is obvious, the notation $\pi_{S,\mathcal{B}}$ is simply replaced by $\pi_{S}$. In practice:
\begin{equation}
\forall (\lambda_i)_{1 \leq i \leq \dim(V)} \in \mathbb{R}^{\dim(V)}, \quad \pi_{S} \left(  \sum_{i=1}^{\dim(V)} \lambda_i e_i \right) = \sum_{i\in S} \lambda_i e_i 
\end{equation} Therefore, from a numerical point of view, one can interpret the feature selector as linear map $\pi_{S}: V \rightarrow \textrm{span}\left(\{ e_i \}_{i \in S} \right)$, which enables reducing the size of the vector representing $\pi_{S}(x)$ for $x\in V$. In this way, applying a feature selector $\pi_{S}$ to a vector of $\mathbb{R}^\mathcal{N}$ consists in masking its features whose indexes are not in $S$, which gives a reduced vector in $\mathbb{R}^{|S|}$ where $|S|$ denotes the number of elements in $S$. Feature selection algorithms build the set $S$ by searching for the most relevant features for the prediction of the output variable $Y$. For this purpose, the mutual information can be used to quantify the degree of the relationship between variables:

\begin{definition}
\textit{(Mutual information~\cite{cover2012elements}, eq. 8.47, p. 251)} Let $Z^1$ and $Z^2$ be two real-valued random variables with joint probability distribution $p_{1,2}$ and marginal distributions $p_1$ and $p_2$. The \textit{mutual information} $I(Z^1, Z^2)$ is defined by:\begin{equation}
I \left( Z^1, Z^2 \right) = \int_{\mathbb{R}^2} p_{1,2}(z^1, z^2) \log \left( \frac{p_{1,2}(z^1, z^2)}{p_{1}(z^1) p_{2}(z^2)} \right) dz^1 dz^2
\end{equation}
\end{definition}

The mutual information measures the mutual dependence between two random variables. Contrary to correlation coefficients, the information provided by this score function is not limited to linear dependence. The mutual information is nonnegative, and equals to zero if and only if the random variables are independent. Given Equation~\eqref{jointDensity}, replacing $Z^1$ by a feature $X^i$ of $X$ and $Z^2$ by $Y$ gives: \begin{equation}
I \left( X^{i}, Y \right) = \sum_{k=1}^{K} \Prob_{Y}(k) \int_{x^{i}\in\mathbb{R}} p_{X^{i} | Y}(x^{i} | k) \log \left( \frac{p_{X^{i} | Y}(x^{i} | k)}{p_{X^{i}}(x^{i})} \right) dx^{i}
\end{equation}

\noindent The mutual information can be used to quantify the redundancy of a set of features $S$ and its relevance for predicting $Y$:

\begin{definition}
\textit{(Relevance~\cite{mRMR2005}, eq. 4, p. 2)} Let $X=(X^i)_{1 \leq i \leq \mathcal{N}}$ be a multivariate random variable, and let $Y$ be a discrete random variable. The \textit{relevance} of a reduced set $S \subset [\![ 1; \mathcal{N} ]\!]$ of features of $X$ for predicting $Y$ is defined by: \begin{equation}
D(S,Y) = \frac{1}{\vline S \vline} \sum_{i\in S} I(X^i , Y)
\end{equation}
\end{definition}

\begin{definition}
\textit{(Redundancy~\cite{mRMR2005}, eq. 5, p. 2)} Let $X=(X^i)_{1 \leq i \leq \mathcal{N}}$ be a multivariate random variable. The \textit{redundancy} of a reduced set $S \subset [\![ 1; \mathcal{N} ]\!]$ of features of $X$ is defined by; \begin{equation}
R(S) = \frac{1}{\vline S \vline ^{2}} \sum_{i,j\in S^2} I(X^i , X^j)
\end{equation}
\end{definition}

\noindent The \textit{minimum redundancy maximum relevance} (mRMR) algorithm~\cite{mRMR2003, mRMR2005} builds the set $S$ by maximizing $D(S,Y) - R(S)$, which is a combinatorial optimization problem. For this type of optimization problem, a brute-force search is intractable, because the number of solution candidates is too large. Instead, mRMR searches for a sub-optimal solution by following a greedy approach. First, the feature having the highest mutual information with the label variable $Y$ is selected. Then, the algorithm follows an incremental procedure: given the set $S_{m-1}$ obtained at iteration $m-1$, form the set $S_m$ such that:
\begin{equation}
S_m = S_{m-1} \cup \left\lbrace \underset{i\in [\![ 1; \mathcal{N} ]\!] \setminus S_{m-1}}\argmax \ \left( I(X^i , Y) - \frac{1}{m-1} \sum_{j \in S_{m-1}} I(X^i, X^j) \right) \  \right\rbrace
\end{equation}
\noindent This incremental procedure stops when $m$ reaches the target number of features $N_f$. A review of feature selection algorithms based on mutual information can be found in~\cite{MIbasedFSreview}.

\subsection{A geostatistical variant of mRMR feature selection}

When training dictionary-based ROM-nets, the number of features of the random vector $X$ scales with the number of nodes $\mathcal{N}$ in the mesh. In particular, the number of features is exactly $\mathcal{N}$ if $X$ is the nodal representation of a scalar field. Hence, there are too many features to compute all redundancy terms $I(X^i, X^j)$. However, one can estimate the redundancy terms thanks to the proximities of the features on the mesh. Indeed, $X$ is a regionalized variable: in our example, we recall that $\bs{\xi}_{i} \in \mathbb{R}^{3}$ denotes the position of the $i$-th node in the mesh, and that the feature $X^i$ corresponds to the value taken by a random temperature field at $\bs{\xi}_{i}$. If two points $\bs{\xi}_{i}$ and $\bs{\xi}_{j}$ of the mesh are close to each other, the corresponding features $X^i$ and $X^j$ are likely to be correlated and thus redundant because of the smoothness of the temperature field. This idea is also valid when considering physical variables discretized in time. 

In this paper, the random temperature field is modeled by a Gaussian random field~\cite{abrahamsen1997review} as in~\cite{ROM-net}, which is a common and simple approach when modeling uncertainties on a physical field. As a consequence, $X$ is a Gaussian random vector and the mutual information $I(X^i, X^j)$ has a simple formula involving the correlation coefficient: 

\begin{proper}
\textit{(Mutual information of two correlated Gaussian random variables~\cite{cover2012elements}, eq.~8.56, p.~252)} Let $(X^1, X^2)$ be a Gaussian random vector. The mutual information $I(X^1, X^2)$ reads:\begin{equation}
I(X^1, X^2) = -\frac{1}{2} \ln \left( 1 - \rho^2  \right)
\label{MIgaussianCase}
\end{equation} \noindent where $\rho$ denotes the correlation between $X^1$ and $X^2$.
\end{proper}

This property implies that, for Gaussian random fields having isotropic correlation functions\footnote{The correlation function $\rho(\bs{\xi}, \bs{\xi}')$ of a random field is isotropic if it only depends on the distance $|| \bs{\xi} - \bs{\xi}' ||_2$.} $\rho$, the mutual information $I(X^i , X^j )$ only depends on the distance $|| \bs{\xi}_{i} - \bs{\xi}_{j} ||_2$. A wide variety of isotropic correlation functions are given in~\cite{abrahamsen1997review}. More generally, since Equation~\eqref{MIgaussianCase} is an increasing function of $\rho^2$, any isotropic upper (resp. lower) bound of the squared correlation function gives an isotropic upper (resp. lower) bound of the mutual information.

\begin{figure}[!h]
\centering
\includegraphics[scale=0.6]{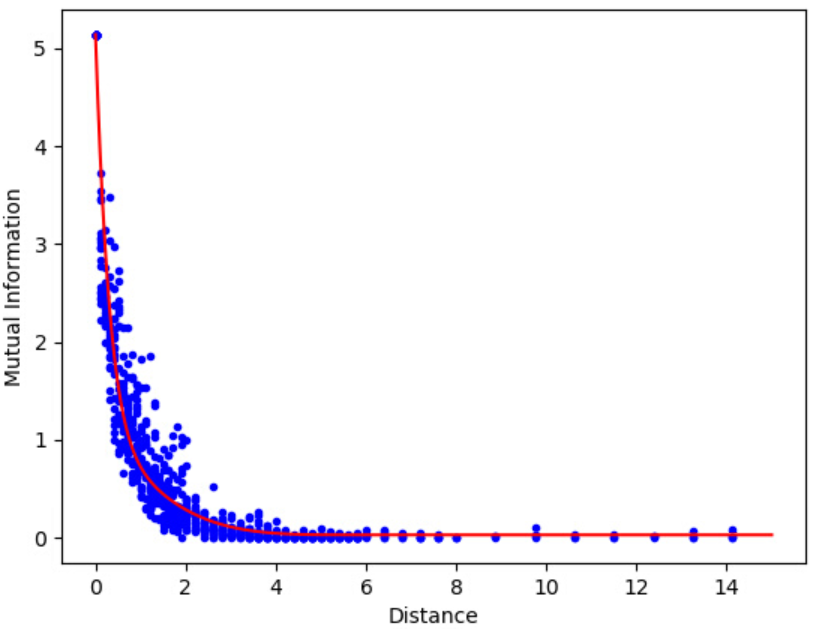}
\caption{Mutual information $I(X^i,X^j)$ as a function of the distance $|| \bs{\xi}_{i} - \bs{\xi}_{j} ||_2$.}
\label{MIvsDist}
\end{figure}

For the example studied in this paper, Figure~\ref{MIvsDist} shows that the mutual information $I(X^i , X^j )$ decreases as the corresponding distance $|| \bs{\xi}_{i} - \bs{\xi}_{j} ||_2$ increases. Therefore, our feature selection algorithm builds a metamodel $\tilde{I}$ replacing $I(X^i, X^j)$ by a function of the distance $|| \bs{\xi}_{i} - \bs{\xi}_{j} ||_2$, which drastically reduces the computational cost of mRMR algorithm for our particular problem. First of all, one must build a design of experiments (DOE) to select a few terms $I(X^i, X^j)$ to be computed exactly. The metamodel $\tilde{I}$ is calibrated to fit the corresponding precomputed redundancy terms. Then, mRMR feature selection is applied by replacing $I(X^i, X^j)$ by $\tilde{I}( || \bs{\xi}_{i} - \bs{\xi}_{j} ||_2 )$. The feature selection algorithm is described in Algorithm~\ref{alg:AlgoFS}. We call this algorithm \textit{geostatistical mRMR}, since geostatistics is the branch of statistics that deals with regonalized variables. A stopping criterion is added to the incremental procedure used in mRMR, enabling an automatic selection of the number of features to be kept for the classification task: the algorithm stops when the value of $\underset{i\in [\![ 1; \mathcal{N} ]\!] \setminus S_{m}}\argmax \left( I(X^i , Y) - \frac{1}{m} \sum_{j \in S_{m}} \tilde{I} \left( \left\|\bs{\xi}_i - \bs{\xi}_j \right\|_{2} \right) \right)$ has not changed much during a number of iterations. A condition on the mutual information $I(X^i, Y)$ can also be added to avoid selecting quasi-irrelevant features. For stage 1 of Algorithm~\ref{alg:AlgoFS}, computing all the terms $\left\|\bs{\xi}_{j_1} - \bs{\xi}_{j_2} \right\|_{2}$ of the matrix of pairwise mesh nodes distances is not necessary: only a few lines of this matrix corresponding to randomly selected nodes are evaluated, which is sufficient to build the DOE. In other words, one computes the distances between a few nodes and all the mesh nodes.

\begin{algorithm}
    \caption{Geostatistical mRMR}
    \label{alg:AlgoFS}
  \begin{algorithmic}[1]
    \INPUT training set $\{ (x_{i},y_{i})\}_{1 \leq i\leq N_{\textrm{train}}}$, set of mesh nodes $\{ \bs{\xi}_{i} \}_{1 \leq i\leq \mathcal{N}}$, stopping criterion. 
    \OUTPUT set of selected features.
    \STATE \textbf{Stage 1 (design of experiments):}
    \Indent
    \STATE Select distance values $r_j$.
    \STATE For each $r_j$, draw $n_j$ pairs of mesh nodes $(\bs{\xi}_{j_1}, \bs{\xi}_{j_2})$ such that $\left\|\bs{\xi}_{j_1} - \bs{\xi}_{j_2} \right\|_{2} \approx r_j$.
    \EndIndent
    \STATE \textbf{Stage 2 (metamodel for redundancy terms):}
    \Indent
    \STATE Compute the mutual information $I(X^i , X^j )$ for each pair selected in Stage 1.
    \STATE Train a metamodel $\tilde{I}$ such that $I(X^i , X^j ) \approx \tilde{I} \left( \left\|\bs{\xi}_{i} - \bs{\xi}_{j} \right\|_{2} \right)$.
    \EndIndent
    \STATE \textbf{Stage 3 (compute relevance terms):}
    \Indent
    \STATE Compute $I(X^i, Y)$ for all $i \in [\![ 1; \mathcal{N} ]\!]$.
    \EndIndent
    \STATE \textbf{Stage 4 (greedy feature selection):}
    \Indent
      \STATE $S_1 := \underset{i\in [\![ 1; \mathcal{N} ]\!]}\argmax \ I(X^i , Y)$
      \STATE $m := 1$
      \WHILE{stopping criterion not satisfied}
        \STATE $S_{m+1} := S_{m} \cup \{ \underset{i\in [\![ 1; \mathcal{N} ]\!] \setminus S_{m}}\argmax \left( I(X^i , Y) - \frac{1}{m} \sum_{j \in S_{m}} \tilde{I} \left( \left\|\bs{\xi}_i - \bs{\xi}_j \right\|_{2} \right) \right) \} $
        \STATE $m := m+1$
      \ENDWHILE
      \EndIndent
      \STATE \textbf{return} $S_m$
  \end{algorithmic}
\end{algorithm}

\begin{remark}
A parallel can be drawn between our feature selection strategy and hyper-reduction methods~\cite{Ryckelynck2005, EIM, ECSW, ECM} used to accelerate complex nonlinear problems in physics (see~\cite{HROMdesignOptim} for design optimization and~\cite{Casenave1} for large-scale simulations). Hyper-reduction methods aim at finding a reduced set of integration points in the finite-element mesh that is sufficient to predict the behavior of the physical system. The constitutive equations are solved on this reduced integration domain only, while the values of quantities of interest at the remaining integration points can be recovered with the Gappy-POD~\cite{GappyPOD}. In short, hyper-reduced solvers make predictions from a reduced number of points in a mesh, like the classifiers used in this paper do when combined with the geostatistical mRMR. Although the objectives are different, both hyper-reduction and geostatistical mRMR feature selection benefit from the properties of physics data to reduce the complexity of numerical tasks.
\end{remark}

\subsection{Numerical results}

The red curve on Figure~\ref{MIvsDist} corresponds to the metamodel estimating redundancy terms. In this example, we choose:\begin{equation}
\tilde{I}(r) = I_{\infty} + \gamma_1 (r_1 - r)^{\alpha_1} H(r_1 - r) + \gamma_2 (r_2 - r)^{\alpha_2} H(r_2 - r)
\end{equation}
\noindent where $H$ is the Heaviside step function and $I_{\infty}, \gamma_1, \gamma_2, r_1, r_2, \alpha_1, \alpha_2$ are calibration parameters that are adjusted manually. In the DOE, the step between distances $r_j$ is smaller for small distances, in order to better capture the evolution of the mutual information in its high gradient regime. The number $n_j$ of pairs of nodes separated by a distance of $r_j$ selected in the DOE also depends on $r_j$: as higher variances were expected for small distances, $n_j$ decreases when $r_j$ increases. In total, $749$ terms $I(X^i, X^j)$ are computed, which takes $5.12$ seconds using Scikit-learn~\cite{scikit-learn}. Building the DOE takes only $0.33$ seconds. Then, the greedy procedure takes $303$ seconds and selects $87$ features among the $42445$ original ones. The first iteration is the longest one with $276$ seconds, because it includes the computation of all the relevance terms $I(X^i, Y)$. As a comparison, the original mRMR algorithm takes $6469$ seconds to compute $7$ iterations only. We did not let mRMR algorithm go further, since the per-iteration computation time grows with the iteration number. For a fair comparison, our implementations of mRMR and stages 3 and 4 of the geostatistical mRMR are the same except that redundancy terms are evaluated with Scikit-learn for mRMR and with the function $\tilde{I}$ for the geostatistical mRMR.

Table~\ref{FSevaluation} compares the relevance $D(S,Y)$, the true redundancy $R(S)$, the approximate redundancy $\tilde{R}(S)$ estimated with $\tilde{I}$, the true cost function $D(S,Y)-R(S)$ and the approximate cost function $D(S,Y)-\tilde{R}(S)$ for three different feature selection strategies:
\begin{itemize}
\item the geostatistical mRMR feature selection (Algorithm~\ref{alg:AlgoFS}), selecting a set $S^*$ of features;
\item a univariate filter algorithm selecting the features with the highest mutual information (MI) scores $I(X^i, Y)$. This algorithms finds a set $S^{MI}$ maximizing the relevance for a given cardinality;
\item a purely geometric feature selection algorithm, randomly selecting the first feature and adding features in a greedy manner so that the distance to the closest point $\bs{\xi}_i, \ i\in S_m$ is maximized. This algorithm tends to select a set $S^G$ of well-distributed features in order to get a low redundancy for a given cardinality.
\end{itemize}
Since the geostatistical mRMR automatically selected $87$ features, the two other approaches are applied with $|S^G|=|S^{MI}|=87$ as a target. Table~\ref{FSevaluation} shows that the relevance of the set $S^*$ selected by our algorithm is in the same order of magnitude as the relevance of the set $S^{MI}$. Its redundancy is in the same order of magnitude as the redundancy of the set $S^G$. These results show that the geostatistical mRMR algorithm does have the desired behavior: it selects a subset of features $S^*$ with high relevance and low redundancy. Figure~\ref{feature_selection} shows the features selected by the three different algorithms. The classification accuracies of several classifiers using the reduced features $S^*$ are given in the last section of the article.

\begin{table}[h!]
\centering
\caption{\textbf{Evaluation of the geostatistical mRMR feature selection algorithm.}}
      \begin{tabular}{cccccc}
        \hline \\[-1em]
        Algorithm & $D(S,Y)$ & $\tilde{R}(S)$ & $R(S)$ & $D(S,Y)-\tilde{R}(S)$ & $D(S,Y)-R(S)$ \\
        \\[-1em] \hline
        Geostatistical mRMR ($S^*$) & $0.0460$  & $0.0816$ & $0.1111$ & $-0.0356$ & $-0.0651$ \\
        MI-based filter ($S^{MI}$) & $0.0671$  & $0.9794$ & $0.8129$ & $-0.9124$ & $-0.7458$ \\
        Geometric filter ($S^G$) & $0.0090$  & $0.0788$ & $0.1072$ & $-0.0699$ & $-0.0982$ \\ \hline
      \end{tabular}
      \label{FSevaluation}
\end{table}

\begin{figure}[!h]
\centering
\includegraphics[scale=0.45]{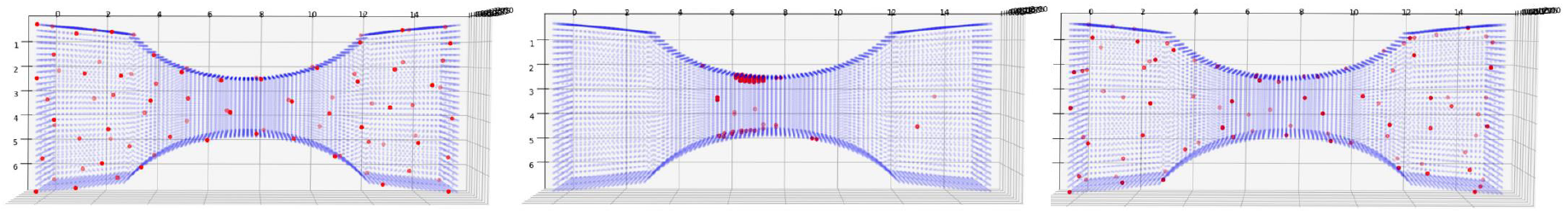}
\caption{Red dots indicate the selected features. From the left to the right: geometric feature selection, MI-based feature selection, geostatistical mRMR.}
\label{feature_selection}
\end{figure}

\begin{remark}
The geometric feature selection algorithm gives rather good results in terms of the cost function, but it does not mean that it is an appropriate approach. Indeed, one can see that the relevance of $S^G$ is very low, since this algorithm does not use any information concerning the classification problem.
\end{remark}

\section{Data augmentation}
\label{SectionDA}

\subsection{Pure sets}

\begin{definition}
\textit{(Convex set~\cite{rockafellar1970convex}, p.~10)} Let $V$ be a real vector space. A non-empty set $\mathcal{S} \subset V$ is convex if:
\begin{equation}
\forall ( x_1 , x_2 ) \in \mathcal{S}^2 , \ \forall \lambda \in [0;1], \quad \lambda x_1 + (1-\lambda) x_2 \in \mathcal{S}
\end{equation}
\end{definition}

\begin{definition}
\textit{(Convex combination~\cite{rockafellar1970convex}, p.~11)} Let $\{ x_i \}_{1 \leq i \leq n}$ be a finite set of elements of a real vector space $V$. A convex combination of $\{ x_i \}_{1 \leq i \leq n}$ is a vector $x\in V$ such that:
\begin{equation}
\exists \ ( \lambda_i )_{1 \leq i \leq n} \in \mathbb{R}_{+}^n \quad | \quad \sum_{i=1}^{n} \lambda_i = 1 \quad \textrm{and} \quad x = \sum_{i=1}^{n} \lambda_i x_i
\end{equation}
\end{definition}

\begin{definition}
\textit{(Convex hull of a set~\cite{rockafellar1970convex}, p.~12)} Let $V$ be a real vector space and $\mathcal{S}$ a non-empty set included in $V$. The convex hull or convex envelope $\mathcal{E}(\mathcal{S})$ of $\mathcal{S}$ is the smallest convex set containing $\mathcal{S}$. Equivalently, the convex hull $\mathcal{E}(\mathcal{S})$ can be defined as the set of all convex combinations of all finite subsets of $\mathcal{S}$.
\end{definition}

\begin{proper}
\textit{(Image of a convex hull by a linear map)} Let $V$ and $W$ be two real vector spaces, and let $\mathcal{L}: V \rightarrow W$ be a linear map. Let $\mathcal{S}$ be a non-empty set included in $V$. Then:
\begin{equation}
\mathcal{L} \left( \mathcal{E}(\mathcal{S}) \right) = \mathcal{E} \left( \mathcal{L}(\mathcal{S}) \right)
\end{equation}
\end{proper}

\begin{proof}
Let $z\in\mathcal{E}\left(\mathcal{L}(\mathcal{S})\right)$. Following the definition of a convex hull, there exists $n\in\mathbb{N}^{*}$ such that: \begin{equation}
\exists \ (w_i)_{1\leq i \leq n}\in \mathcal{L}(\mathcal{S})^{n}, \ \exists \ (\lambda_i)_{1\leq i \leq n}\in \mathbb{R}_{+}^{n} \quad | \quad \sum_{i=1}^{n} \lambda_i =1 \ \textrm{and} \ z = \sum_{i=1}^{n} \lambda_i w_i
\end{equation} For all $i\in [\![ 1;n ]\!]$, as $w_i \in \mathcal{L}(\mathcal{S})$, there exists $v_i \in S$ such that $w_i = \mathcal{L}(v_i)$. By linearity of $\mathcal{L}$: \begin{equation}
z = \sum_{i=1}^{n} \lambda_i \mathcal{L}(v_i) = \mathcal{L} \left( \sum_{i=1}^{n} \lambda_i v_i \right) \in\mathcal{L}\left(\mathcal{E}(\mathcal{S})\right)
\end{equation} so $\mathcal{E}\left(\mathcal{L}(\mathcal{S})\right) \subset \mathcal{L}\left(\mathcal{E}(\mathcal{S})\right)$. The other inclusion can be shown using exactly the same arguments. Thus: $\mathcal{L} \left( \mathcal{E}(\mathcal{S}) \right) = \mathcal{E} \left( \mathcal{L}(\mathcal{S}) \right)$.
\end{proof}

\noindent This property has a very simple yet important consequence for the data augmentation algorithm presented in this paper:

\begin{proper}
Let $V$ and $W$ be two real vector spaces, and let $\mathcal{L}: V \rightarrow W$ be a linear map. Let $\mathcal{S}$ be a non-empty set included in $V$. Then, for all $x\in V$: \begin{equation}
\mathcal{L}(x) \notin \mathcal{E} \left( \mathcal{L} (\mathcal{S}) \right) \Rightarrow x \notin \mathcal{E} (\mathcal{S})
\end{equation}
\label{PropertyLinMap}
\end{proper}

\begin{proof}
By contraposition, $x \in \mathcal{E} (\mathcal{S}) \Rightarrow \mathcal{L}(x) \in \mathcal{L} \left( \mathcal{E} (\mathcal{S}) \right) = \mathcal{E} \left( \mathcal{L} (\mathcal{S}) \right)$. 
\end{proof}

\noindent Our data augmentation strategy uses this property in the particular case where the linear map is a projection. As a reminder, the notation $\mathcal{K}$ stands for the true classifier assigning any input $x$ to a single label $y \in [\![ 1;K ]\!]$. Before giving the description of the algorithm, let us introduce the definition of \textit{pure sets} in a labeled dataset and a characterization theorem:

\begin{definition}
\textit{(Pure set)} Let $n$ be a positive integer, and let $\mathcal{S} = \{ x_i \}_{1 \leq i \leq n}$ be a finite set of elements of a real vector space $V$ labeled by $\mathcal{K}$. Let $\mathcal{S}_{\mathcal{I}} = \{ x_i \}_{i \in \mathcal{I} \subset [\![ 1;n ]\!]}$ be a non-empty subset of $\mathcal{S}$. The set $\mathcal{S}_{\mathcal{I}}$ is pure in $\mathcal{S}$ if $\mathcal{K} \left( \mathcal{S} \cap \mathcal{E}(\mathcal{S}_{\mathcal{I}}) \right)$ is a singleton, which means that the set $\mathcal{S}_{\mathcal{I}}$ is pure in $\mathcal{S}$ if all of the points of $\mathcal{S}$ that belong to the convex hull of $\mathcal{S}_{\mathcal{I}}$ have the same label.
\end{definition}

Let $\mathcal{S} = \{ x_i \}_{1 \leq i \leq n}$ be a finite set of elements of a finite-dimensional real vector space $V$ labeled by integers $\{ y_i \}_{1 \leq i \leq n}$ in $[\![ 1;K ]\!]$, with $K \leq n$. For all $k \in [\![ 1;K ]\!]$, $C_k$ denotes the set of elements of $\mathcal{S}$ labeled by $k$:\begin{equation}
C_k = \{ x_i \in \mathcal{S} \ | \ y_i = k \}
\end{equation}

\noindent For any subset $S_k$ of $C_k$ with cardinality $| S_k |$, $\hat{A}_{S_k} \in \mathbb{R}^{\dim(V)\times | S_k |}$ denotes the matrix whose columns contain the coordinates of the elements of $S_k$. The matrix denoted by $A_{S_k}$ is obtained by adding a row of ones below the matrix $\hat{A}_{S_k}$, giving a matrix of size $(1+\dim(V))\times | S_k |$.

\begin{thm}
\textit{(Pure set characterization)} Let $S_k$ be a subset of $C_k$ with cardinality $| S_k |$. The set $S_k$ is pure in $\mathcal{S}$ if and only if for all $x$ in $\mathcal{S} \setminus C_k$ the linear system:\begin{equation}
A_{S_k} w = \begin{pmatrix} x \\ 1 \end{pmatrix}
\label{LinSystemPureSet}
\end{equation}
\noindent has no nonnegative solution $w\in\mathbb{R}_{+}^{| S_k |}$.
\label{ThmPureSetCarac}
\end{thm}

\begin{proof}
Let $x \in \mathcal{S} \setminus C_k$. Equation~\eqref{LinSystemPureSet} has no nonnegative solution if and only if:\begin{equation}\nexists \ w \in \mathbb{R}_{+}^{| S_k |} \quad | \quad  \sum_{i=1}^{| S_k |} w_i = 1 \ \textrm{and} \ \hat{A}_{S_k} w = x
\end{equation}\begin{equation}\iff x \notin \mathcal{E} \left( S_k \right)
\end{equation} \noindent which ends the proof.\end{proof}

\begin{cor}
\textit{(Pure set testing)} Let $S_k$ be a subset of $C_k$ with cardinality $| S_k |$, and let $\mathcal{L}: V \rightarrow W$ be a linear map, where $W$ is a finite-dimensional real vector space. If for all $x$ in $\mathcal{S} \setminus C_k$ the linear system:\begin{equation}
A_{\mathcal{L}(S_k)} w = \begin{pmatrix} \mathcal{L}(x) \\ 1 \end{pmatrix}
\label{LinSystemPureSetTest}
\end{equation}
\noindent has no nonnegative solution in $\mathbb{R}_{+}^{| S_k |}$, then $S_k$ is pure in $\mathcal{S}$.
\label{PureSetTesting}
\end{cor}

\begin{proof}
Equation~\eqref{LinSystemPureSetTest} characterizes the purity of $\mathcal{L}(S_k)$ in $\mathcal{L}(\mathcal{S})$ (Theorem~\ref{ThmPureSetCarac}), which implies that $S_k$ is pure in $\mathcal{S}$ (Property~\ref{PropertyLinMap}).
\end{proof}

Figure~\ref{fig:FigPureSets} illustrates the concept of pure sets. On this figure, the set $C_1$ is made of all the elements represented by dots, while the crosses form the set $C_2 = \mathcal{S} \setminus C_1$. On the left, the subset formed by the six black dots is pure since its convex hull delimited by dashed lines contains only dots. The subset made of the six black dots on the right-hand side of the figure is not pure because of the presence of a cross in its convex hull. Equation~\eqref{LinSystemPureSet} has a nonnegative solution when using the coordinates of this cross in its right-hand side.

\begin{figure}[h]
\begin{center}
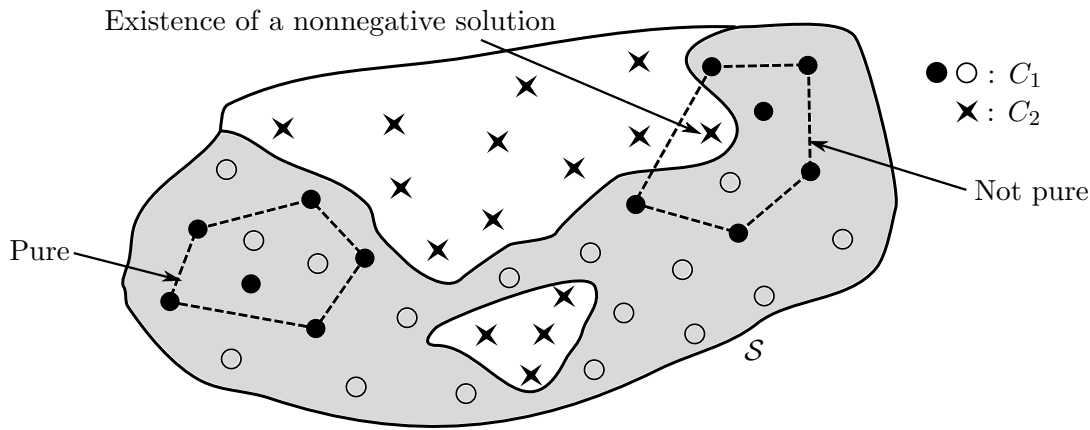
\end{center}
\caption{Illustration of the concept of pure sets on a binary classification problem.}
\label{fig:FigPureSets} 
\end{figure}

\subsection{The data augmentation algorithm}

The objective is to generate new data points $x\in\mathcal{X}$ in a given class $y\in [\![ 1;K ]\!]$ from the preexisting observations in that class. To this end, one must apply class-preserving transformations on the training examples. New examples can be created by taking convex combinations of some subsets of the training set, for example. One way of controlling the risk that newly generated examples have wrong labels is to take convex combinations of subsets only if they are pure. Indeed, if the $k$-th class $C_k$ contains a set $S_k$ that is pure in the training set, one can expect that the probability $\Prob( Y = k \ | \ X \in \mathcal{E}\left( S_k \right) )$ is high enough to get new examples of class $C_k$ by drawing samples in $\mathcal{E}\left( S_k \right)$. In addition, the third Hadamard well-posedness condition states that the solution of a physics equation changes continuously with respect to the parameters of the problem. In the neighborhood of a point $x_0$ belonging to a pure set $S_k$, the primal solution $u$ stays in the neighborhood of the solution $u_0$ obtained with $x_0$ and is thus likely to have the same label. Hence, the objective of our algorithm is to find pure sets in the training set in order to generate new examples by convex combinations with a limited risk of getting incorrectly labeled examples. The pure sets detected by the algorithm are listed in a matrix $\mathcal{S}$ such that $\mathcal{S}[k,i]$ contains the indices of the training points forming the $i$-th pure set of the $k$-th class. The pure sets are grown from different starting points or seeds in the training set by iteratively adding the seeds' nearest neighbors in terms of the precomputed dissimilarity measure $\delta$ used for clustering in the data labeling procedure. The growth stops before losing the purity of the subsets. However, checking the purity in the high-dimensional input space can cause difficulties, even when training the data augmentation algorithm after a first dimensionality reduction like in this paper. For this reason, the algorithm checks the purity after having applied a feature selector $\pi_S$ with a small random subset of features $S$ containing $d$ features. Let us apply Property~\ref{PropertyLinMap} with $V=W$ being the input vector space containing $\mathcal{X}$ and with the linear map $\mathcal{L}$ being the feature selector $\pi_S$. As Property~\ref{PropertyLinMap} states, if no point of $\pi_{S}\left( \{ x_m \}_{1 \leq m \leq N_{\textrm{train}}} \setminus C_k \right)$ belongs to the convex hull of $\pi_{S}\left( \{ x_m \}_{m\in\mathcal{S}[k,i]} \right)$, then the set $\mathcal{E}\left( \{ x_m \}_{m\in\mathcal{S}[k,i]} \right)$ does not contain any point labeled with $k' \neq k$. Since a set can lose its purity after projection, the algorithms tries $p_{\max}$ random feature selectors $\pi_S$ before considering that the set is not pure. In practice, the purity of $\pi_{S}\left( \{ x_m \}_{m\in\mathcal{S}[k,i]} \right)$ in $\pi_{S}\left( \{ x_m \}_{1 \leq m \leq N_{\textrm{train}}} \right)$ is numerically tested by solving a nonnegative least squares (NNLS~\cite{doi:10.1137/1.9781611971217}) problem. If for all points $x \in \{ x_m \}_{1 \leq m \leq N_{\textrm{train}}} \setminus C_k$ the inequality:\begin{equation}
\underset{w \in \mathbb{R}_{+}^{| \mathcal{S}[k,i] |}}{\min} || A_{\pi_{S}\left( \{ x_m \}_{m\in\mathcal{S}[k,i]} \right)} w - \tilde{\pi}_{S}(x) ||_2 \geq \varepsilon_{\textrm{DA}} || \tilde{\pi}_{S}(x) ||_2
\end{equation}
\noindent is satisfied with $\tilde{\pi}_{S}(x) = ( \pi_{S}(x)^T \ 1 )^T$ and with $\varepsilon_{\textrm{DA}}$ being the tolerance of the data augmentation algorithm, then Corollary~\ref{PureSetTesting} implies that $\{ x_m \}_{m\in\mathcal{S}[k,i]}$ is pure in $\{ x_m \}_{1 \leq m \leq N_{\textrm{train}}}$. Algorithm~\ref{alg:DAtrain} describes the data augmentation algorithm. It calls Algorithm~\ref{alg:DAinit} to find $n$ well-distribued seeds per class before growing pure sets. It is noteworthy that using few pure sets to generate many examples would increase the distribution gap~\cite{distribGap} between augmented data and original data. To avoid this issue, one had better use many well-distributed seeds to distribute data augmentation efforts between the pure sets and thus limit the divergence between the augmented distribution and the true data-generating distribution.

\begin{remark}
Realizations of the random variable $X$ belong to a convex domain $\mathcal{X}$ related to physics constraints. When considering surface random temperature fields defined on the boundaries of a solid, $\mathcal{X}$ is a hypercube consisting of all the fields taking values between zero Kelvin degree and the material's melting point. These random fields can be used as Dirichlet boundary conditions for the nonlinear heat equation. The assumption of a linear thermal behavior is added when considering three-dimensional random temperature fields defined inside the solid, so that the set $\mathcal{X}$ remains convex when adding the constraint that the random field must satisfy the heat equation. More generally, convex combinations respect physics constraints defined by linear operators, such as linear partial differential equations and Dirichlet, Neumann and Robin boundary conditions.
\end{remark} 

\begin{algorithm}
    \caption{Seeds selection for data augmentation. Note: all the dissimilarities have already been computed in the data labeling procedure.}
    \label{alg:DAinit}
  \begin{algorithmic}[1]
    \INPUT training set $\{ (x_{i},y_{i})\}_{1 \leq i\leq N_{\textrm{train}}}$, class label $k$, class center $\tilde{x}_k$, dissimilarity matrix $\bs{\delta}$, target number of seeds $n$, preselection parameters $(\varepsilon_1 , \varepsilon_2 ) \in [0;1]^2$. 
    \OUTPUT List $l_k$ of $n$ indices of seeds candidates for the $k$-th class.
    \STATE \textbf{Stage 1 (filter the data):}
    \STATE Find the minimum dissimilarity $\delta_{\textrm{ref}}^k$ separating the class center $\tilde{x}_k$ from a point belonging to another class.
    \STATE Remove points having neighbors belonging to foreign classes within a distance of $\varepsilon_1 \delta_{\textrm{ref}}^k$.
    \STATE Remove isolated points having no neighbor within a distance of $\varepsilon_2 \delta_{\textrm{ref}}^k$.
    \STATE $\mathcal{I}_{k} := $ set of the indices of the remaining points in class $k$.
    \STATE \textbf{Stage 2 (maximin greedy selection):}
    \STATE Initialize $l_k$ with the index of the class center $\tilde{x}_k$.
    \FOR{$i \in [\![ 2; \min(n, |\mathcal{I}_{k}|-1) ]\!]$}
      \STATE $j := \underset{l \in \mathcal{I}_{k} \setminus l_k}{\argmax} \ \ \underset{m \in l_k}{\min} \ \delta_{lm}$
      \STATE Append $j$ to $l_k$.
    \ENDFOR
    \STATE \textbf{return} $l_k$
  \end{algorithmic}
\end{algorithm}

\begin{algorithm}
    \caption{Data augmentation algorithm}
    \label{alg:DAtrain}
  \begin{algorithmic}[1]
    \INPUT training set $\{ (x_{i},y_{i})\}_{1 \leq i\leq N_{\textrm{train}}}$, dissimilarity matrix $\bs{\delta}$, per-class number of seeds $n$, maximum number of pure set testings $p_{\textrm{max}}$, dimension $d$ of subspaces for pure set testings, number of augmented data $N_{DA}$. 
    \OUTPUT augmented data $\{ (\tilde{x}_{i},\tilde{y}_{i})\}_{1 \leq i\leq N_{DA}}$ and matrix $\mathcal{S}$ listing pure sets.
    \STATE \textbf{Stage 1 (find pure sets in the training set):}
    \FOR{$k \in [\![ 1;K ]\!]$}
      \STATE Apply Algorithm~\ref{alg:DAinit} to get the list $l_k$ of $n$ indices of seeds candidates.
      \FOR{$i \in [\![ 1;n ]\!]$}
        \STATE $\mathcal{S}_{1} := \{ l_{k}[i] \}$
        \STATE $\textrm{neighbors} := \textrm{argsort}(\bs{\delta}[l_{k}[i],:])$
        \STATE $j := 1$
        \STATE $\textrm{setPurity} := \textrm{True}$
        \WHILE{$\textrm{setPurity}$}
          \STATE $\mathcal{S}_{j+1} := \mathcal{S}_{j} \cup \{ \textrm{neighbors}[j] \}$
          \STATE $j := j+1$
          \STATE $p := 1$
          \STATE Select a random subset $S$ of $d$ features.
          \WHILE{$\{ \pi_{S}(x_m) \}_{m \in \mathcal{S}_{j}}$ is not pure in $\{ \pi_{S}(x_m) \}_{1 \leq m \leq N_{\textrm{train}}}$ and $p \leq p_{\textrm{max}}$}
            \STATE Select a new random subset $S$ of $d$ features.
            \STATE p := p+1
          \ENDWHILE
          \IF{$p = p_{\textrm{max}}+1$}
            \STATE $\textrm{setPurity} := \textrm{False}$
          \ENDIF
        \ENDWHILE
        \STATE $\mathcal{S}[k,i] := \mathcal{S}_{j-1}$
      \ENDFOR
    \ENDFOR
  \STATE \textbf{Stage 2 (generate new data):}
  \STATE Generate $N_{DA}$ random convex combinations $\{\tilde{x}_{i}\}_{1 \leq i\leq N_{DA}}$ of the pure sets listed in $\mathcal{S}$. Convex combinations $\tilde{x}_{i}$ of the pure set described in $\mathcal{S}[k,j]$ are labeled by $\tilde{y}_{i}=k$.
  \STATE \textbf{return} $\{ (\tilde{x}_{i},\tilde{y}_{i})\}_{1 \leq i\leq N_{DA}}$ and $\mathcal{S}$
  \end{algorithmic}
\end{algorithm}

\subsection{Numerical results}

Linear discriminant analysis (LDA), commonly used for classification tasks, can also be used for supervised dimensionality reduction by projecting the data onto the subspace maximizing the between-class variance, as explained in~\cite{Hastie2005TheEO}. For the classification problem presented in this paper, the training data are visualized in the two-dimensional subspace obtained by LDA in Figure~\ref{clusteringLDAvisu}. Although this subspace is the one that best separates the classes, one can see that the training examples do not form well-separated groups. For this reason, testing the purity of subsets of training data before generating new examples by convex combinations is necessary to reduce the risk of getting incorrectly labeled augmented data.

\begin{figure}[!h]
\centering
\includegraphics[scale=0.6]{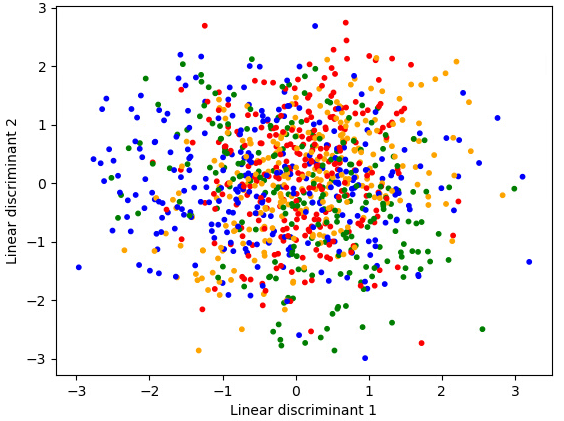}
\caption{Data visualization in the 2D subspace maximizing the separation between classes (supervised linear dimensionality reduction using LDA).}
\label{clusteringLDAvisu}
\end{figure}

The data augmentation algorithm finds about 60 pure sets per class with an average population of 5 training examples, using random subspaces of dimension 5 to test the purity. Note that two pure sets are merged only when one is included in the other, since the union of two pure sets is not always pure. The computation time for the data augmentation training phase is 40 minutes. Once pure sets have been found, one can generate as many augmented examples as necessary. Generating 5400 examples to multiply the size of the training set by 10 takes less than a second. Among the augmented data, 400 examples are taken for the evaluation of the data augmentation algorithm. The data labeling procedure involving numerical simulations is applied for these 400 examples in order to estimate the percentage of incorrectly labeled data. It turns out that none of these examples is incorrectly labeled, which validates the algorithm for our problem. The benefits of data augmentation for the classification task are evaluated in the final section.

\section{Validation of our feature selection and data augmentation algorithms}
\label{SectionValidation}

\subsection{Classification performances of various classifiers}

In this section, 14 different classifiers are trained and evaluated on our classification problem. To evaluate whether the features selected by geostatistical mRMR are relevant for classification purposes, each classifier is tested twice: once in combination with the geostatistical mRMR and once with principal component analysis (PCA) with 10 modes. Since the random temperature fields derive from a Gaussian random field involving only 10 modes, the database obtained after applying PCA contains all the information of the original data. Each combination of one of the 14 classifiers with PCA or feature selection is trained twice: once on the true training set containing $N_{\textrm{train}}=600$ examples, and once on the augmented training set made of $6000$ examples.

All the classifiers are trained with Scikit-learn~\cite{scikit-learn}, except multilayer perceptrons (MLPs) and radial basis function networks (RBFNs) which are trained with PyTorch~\cite{PyTorch_NIPS2019_9015}. We train the RBFNs in a fully supervised manner, which means that the parameters of the radial basis functions are learnt by gradient descent like the weights of the fully-connected layers. In addition, we use only one hidden layer for RBFNs, since these artificial neural networks generally have shallow architectures, as explained in~\cite{Aggarwal18}. Deeper architectures have been tested for MLPs. Scikit-learn's MLP classifier has also been tested; it is called \textit{simple MLP} in this paper, because its architecture is only made of fully-connected layers and does not include dropout~\cite{dropout2014} nor batch normalization~\cite{batchNorm}. All the classifiers based on artificial neural networks are trained with Tikhonov regularization and early stopping~\cite{Goodfellow-et-al-2016}. Logistic regression is trained with elastic net regularization~\cite{10.2307/3647580}. Kernels used for support vector machines (SVMs) are obtained by combining several polynomial kernels with different hyperparameters. Kernel design could be optimized using multiple kernel learning algorithms~\cite{MultipleKernelLearning}, but we simply build our kernels by evaluating different combinations on the validation set, just as when we look for a good architecture for artificial neural networks.

The classification accuracies on test data are given in Table~\ref{ClassificationResults}. Of course, this ranking is specific to the classification problem presented in this paper, no general conclusion can be drawn from this particular numerical application. On this classification problem, when using augmented data in the training phase, the highest test accuracy reached with linear classifiers is $43.5\%$, obtained with the linear SVM combined with PCA. The fact that k-nearest neighbors classifiers barely exceed $50.0\%$ of accuracy on this problem is related to an observation that was made in~\cite{ROM-net}: there is no simple correlation between the Euclidean distance and the physics-informed dissimilarity measure used in dictionary-based ROM-nets. MLPs get the best results, reaching $87.0\%$ of accuracy when combined with our data augmentation and feature selection algorithms. Interestingly, quadratic discriminant analysis (QDA) gives excellent results while having no hyperparameter to tune, contrary to the two other families of classifiers obtaining the best results, namely MLPs and multiple kernel SVMs. This makes QDA the best compromise between accuracy and training complexity for this specific classification task.

Although PCA perfectly describes the input data in this example, the geostatistical mRMR feature selection algorithm enables reaching higher accuracies with some classifiers. Not only it behaves as the original mRMR when selecting features, but it also gives satisfying results when combined with a classifier. Concerning data augmentation, Table~\ref{ClassificationResults} shows that our algorithm significantly improves classification results. The accuracy gain due to data augmentation is $4.98\%$ on average and ranges from $-2.5\%$ to $10.5\%$, increasing the accuracy in 25 cases out of 28.

\begin{table}[h!]
\centering
\caption{\textbf{Test accuracies of different classifiers with dimensionality reduction via principal component analysis (PCA) or feature selection (FS), with and without data augmentation (DA).}}
      \begin{tabular}{|cccc|}
        \hline
        Classifier & Dim. red. & Acc. with DA  & Acc. without DA \\ \hline
        Stacking (6 MLPs and logistic regression) & FS  & $90.0\%$   & $-$ \\
        Ensemble averaging (6 MLPs) & FS  & $89.0\%$   & $-$ \\ \hline
        Multilayer perceptron & FS  & $87.0\%$   & $81.0\%$ \\
        Multilayer perceptron & PCA & $86.5\%$ & $81.5\%$ \\
        Simple multilayer perceptron & PCA & $85.0\%$ & $79.5\%$ \\
        Simple multilayer perceptron & FS & $84.0\%$ & $80.0\%$ \\
        Quadratic discriminant analysis & FS & $77.5\%$ & $70.5\%$ \\
        Quadratic discriminant analysis & PCA & $76.0\%$ & $70.0\%$ \\
        Multiple kernel support vector machine & PCA & $73.0\%$ & $68.0\%$ \\
        	Multiple kernel support vector machine & FS & $72.5\%$ & $66.0\%$ \\
        Random forest & FS & $69.0\%$ & $63.0\%$ \\
        AdaBoost & FS & $68.5\%$ & $63.0\%$ \\
        Gradient-boosted trees & FS & $68.0\%$ & $58.5\%$ \\
        Radial basis function network & PCA & $63.5\%$ & $62.0\%$ \\
        Radial basis function network & FS & $62.5\%$ & $60.0\%$ \\
        Decision tree & FS & $55.5\%$ & $43.5\%$ \\
        k-nearest neighbors & PCA & $51.0\%$ & $46.0\%$ \\
        AdaBoost & PCA & $50.5\%$ & $52.5\%$ \\
        k-nearest neighbors & FS & $50.0\%$ & $47.0\%$ \\
        Gradient-boosted trees & PCA & $49.5\%$ & $48.0\%$ \\
        Random forest & PCA & $45.0\%$ & $47.5\%$ \\
        Linear support vector machine & PCA & $43.5\%$ & $33.0\%$ \\
        Linear support vector machine & FS & $40.5\%$ & $34.5\%$ \\
        Gaussian naive Bayes & FS & $39.5\%$ & $34.5\%$ \\
        Gaussian naive Bayes & PCA & $38.5\%$ & $31.5\%$ \\
        Penalized logistic regression & PCA & $38.5\%$ & $28.0\%$ \\
        Penalized logistic regression & FS & $37.0\%$ & $29.0\%$ \\
        Decision tree & PCA & $34.0\%$ & $36.5\%$ \\
        Linear discriminant analysis & PCA & $33.5\%$ & $29.0\%$ \\
        Linear discriminant analysis & FS & $32.5\%$ & $29.0\%$ \\ \hline
      \end{tabular}
      \label{ClassificationResults}
\end{table}

\subsection{How to further improve classification performances?}

Ensemble methods can be used to reduce overfitting and increase the accuracy on test data. In addition, it enables recycling different variants of a classifier that the user has trained for different hyperparameters. Using ensemble averaging with classifiers trained on the augmented dataset with feature selection, we manage to combine 6 MLPs with different architectures to reach an accuracy of $89.0\%$. When stacking these MLPs with a ridge logistic regression analyzing the predicted membership probabilities, we get an accuracy of $90.0\%$. In addition to ensemble learning methods, one can also use random noise injection to increase noise robustness, as explained in~\cite{Goodfellow-et-al-2016}.

\section{Conclusion}

Classification algorithms are used in computational physics for automatic model recommendation. Such modeling strategies enable the reduction of the computation time, or the selection between models with different physics when one wants to improve the accuracy of numerical predictions. This article deals with the specificities of the classification problems encountered in computational physics, and more particularly for dictionary-based ROM-nets. These classification problems generally have the three following issues: the lack of training data, their high dimensionality, and the non-applicability of common data augmentation techniques to physics data. To tackle these difficulties, two algorithms are proposed. The first one is a geostatistical variant of the mRMR feature selection algorithm, enabling the identification of a reduced set of relevant but non-redundant features for high-dimensional regionalized variables. The second one is a data augmentation algorithm controlling the risk of generating new examples with wrong labels by finding pure subsets in the training set. The performances and benefits of these algorithms are illustrated on a classification problem for which 14 classifiers are evaluated.

\section*{Acknowledgements}

The authors wish to thank S\'{e}bastien Da Veiga (SafranTech) for his sound advice, as well as Felipe Bordeu (SafranTech) and Julien Cortial (SafranTech) who implemented the Python library \textit{BasicTools} (https://gitlab.com/drti/basic-tools) with FC.

\section*{Funding}

Study funded by Safran and ANRT (Association Nationale de la Recherche et de la Technologie, France).

\appendix

\bibliographystyle{unsrt}
\bibliography{biblio}

\end{document}